\PassOptionsToPackage{numbers}{natbib}

\documentclass{article}

\usepackage[preprint]{neurips_2019}

\usepackage[utf8]{inputenc} 
\usepackage[T1]{fontenc}    
\usepackage{hyperref}       
\usepackage{url}            
\usepackage{booktabs}       
\usepackage{amsfonts}       
\usepackage{nicefrac}       
\usepackage{microtype}      

\usepackage{amsmath}
\usepackage{algorithm}
\usepackage{algorithmic}
\usepackage{amsthm}

\usepackage{graphicx}
\usepackage{multirow}
\usepackage{array}

\usepackage{natbib}

\newtheorem{assumption}{Assumption}
\newtheorem{lemma}{Lemma}
\newtheorem{theorem}{Theorem}

\title{Efficient Subsampled Gauss-Newton and Natural Gradient Methods for Training Neural Networks}

\author{%
  Yi Ren \\
  IEOR Department \\
  Columbia University\\
  \texttt{yr2322@columbia.edu} \\
  \And
  Donald Goldfarb \\
  IEOR Department \\
  Columbia University\\
  \texttt{goldfarb@columbia.edu} \\
}

\begin{document}

\maketitle

\begin{abstract}
  We present practical Levenberg-Marquardt variants of Gauss-Newton and natural gradient methods for solving non-convex optimization problems that arise in training deep neural networks involving enormous numbers of variables and huge data sets. Our methods use subsampled Gauss-Newton or Fisher information matrices and either subsampled gradient estimates (fully stochastic) or full gradients (semi-stochastic), which, in the latter case, we prove convergent to a stationary point. By using the Sherman-Morrison-Woodbury formula with automatic differentiation (backpropagation) we show how our methods can be implemented to perform efficiently. Finally, numerical results are presented to demonstrate the effectiveness of our proposed methods.
\end{abstract}

\section{Introduction}

First-order stochastic methods are predominantly used to train deep neural networks (NN), including mini-batch gradient descent (SGD) and its variants that use momentum and acceleration \cite{bottou2018optimization, qian1999momentum} and an adaptive learning rate \cite{duchi2011adaptive, kingma2014adam}. 
First-order methods are easy to implement, and only require moderate computational cost per iteration. However, it is cumbersome to tune their hyper-parameters (e.g., learning rate), and they are often slow to escape from regions where the objective function's Hessian matrix is ill-conditioned. Although adaptive learning rate methods converge fast on the training set, pure SGD is still of interest as it sometimes has better generalization property (\cite{keskar2017improving}).

Second-order stochastic methods have also been proposed for training deep NNs because they take far fewer iterations to converge to a solution by using knowledge of the curvature of the objective function. They also have the ability to both escape from regions where the Hessian of objective function is ill-conditioned, and provide adaptive learning rates. Their main drawback is that, due to the huge number of parameters that deep NNs have, it is practically impossible to compute and invert a full Hessian matrix. Efforts to overcome this problem include Hessian-free inexact Newton methods, stochastic L-BFGS methods, Gauss-Newton and natural gradient methods and diagonal scaling methods. See \cite{bottou2018optimization} for a review of these approaches and relevant references.

\subsection{Our Contributions}

Our main contribution is the development of methods for training deep NNs that incorporate partial, but substantial, second-order information, while keeping the computational cost of each iteration comparable to that required by first-order methods. To achieve this, we propose new generic subsampled generalized Gauss-Newton and natural gradient methods that can be implemented efficiently and are provably convergent. Our methods add a Levenberg-Marquardt (LM) damping term to  the Gauss-Newton and Fisher information matrices and invert the resulting matrices using the Sherman-Morrison-Woodbury formula. Moreover, by taking advantage of the Kronecker factored structure in these matrices, we are able to form and invert them in $O(n)$ time. Furthermore, we prove that semi-stochastic versions of our algorithms (i.e., those that use a full gradient combined with mini-batch stochastic Gauss-Newton or Fisher information matrices) converge to a stationary point. 
We demonstrate the effectiveness of our methods with numerical experiment, comparing both first-order method (SGD) and second-order methods (Hessian-free, KFAC). 





\subsection{Closely Related Work}

Our methods were initially motivated by the Hessian-free approach of \cite{martens2010deep}, which approximates the Hessian by the generalized Gauss-Newton matrix and then approximately solves the huge $n \times n$ linear system involving that matrix and an LM damping term to update the $n$ parameters of the NN by an ”early-termination” linear CG method. Other closely related methods include the Krylov subspace descent method of \cite{vinyals2012krylov}, which generalizes the Hessian-free approach by constructing a Krylov subspace; the KFAC method \cite{martens2015optimizing}, which uses the block-diagonal part of the Fisher matrix to approximate the Hessian; and the Kronecker Factored Recursive Approximation method \cite{botev2017practical}, which uses a block-diagonal approximation of the Gauss-Newton matrix. 
For very recent work  on properties of the natural gradient method and the Fisher matrix in the context of NNs see \cite{bernacchia2018exact, cai2019gram, zhang2019fast}.



\section{Background}

\subsection{Feed-forward Neural Networks}

Although our methods are applicable to a wide range of NN architectures, for simplicity, we focus on feed-forward fully-connected NNs with $L+1$ layers. At the $l$-th layer, given the vector of outputs from the preceding layer $v^{(l-1)}$ as input, $v^{(l)}$ is computed as $v^{(l)} = \phi^{(l)}(W^{(l)} v^{(l-1)} + b^{(l)})$, where $W^{(l)} \in \mathbb{R}^{m_l \times m_{l-1}}$, $b^{(l)} \in \mathbb{R}^{m_l}$, and $\phi^{(l)}: \mathbb{R}^{m_l} \to \mathbb{R}^{m_l}$ is a nonlinear activation function. Hence, given the input $x = v^{(0)}$, the NN outputs $\hat{y} = v^{(L)}$. To train the NN we minimize an empirical average loss
\begin{align}
	f(\theta)
    = \frac{1}{N} \sum_{i=1}^N f_i(\theta)
    = \frac{1}{N} \sum_{i=1}^N \varepsilon(\hat{y}_i(\theta), y_i),
    \label{equation_15}
\end{align}
where
$\theta = \left(
	\text{vec}(W^{(1)})^T,
	(b^{(1)})^T,
	\cdots,
    \text{vec}(W^{(L)})^T,
	(b^{(L)})^T
\right)^T$
($\text{vec}({W})$ vectorizes the matrix $W$ by concatenating its columns) and $\varepsilon(\hat{y}_i(\theta), y_i)$ is a loss function based on the differences between $\hat{y}_i$ and $y_i$, for the given set $\{ (x_1, y_1), ..., (x_N, y_N) \}$ of $N$ data points. Note $\theta \in R^n$, where $n = \sum_{l=1}^L \left( m_l m_{l-1} + m_l \right)$ can be extraordinarily large.



\subsection{Approximations to the Hessian matrix}

At iteration $t$, at the point $\theta = \theta^{(t)}$, Newton-like methods compute 
$
    p_t = - G_t^{-1} g_t,
$
where $G_t$ is an approximation to the Hessian of $f(\theta^{(t)})$, and $g_t$ is $\nabla f(\theta^{(t)})$, or an approximation to it, and then set $\theta^{(t+1)} = \theta^{(t)} + p_t$. Computing $G_t$ and inverting it (solving $G_t p_t = - g_t$) is the core step of such methods. Finding a balance between the cost of computing $p_t$ and determining an accurate direction $p_t$ is crucial to developing a good algorithm. 

\subsubsection{Gauss-Newton Method}

In order to get a good approximation to the Hessian of $f(\theta)$, we first examine the Hessian of $f_i(\theta)$ corresponding to a single data point. By (\ref{equation_15}) it follows from the chain rule that
\begin{align*}
    \frac{\partial^2 f_i(\theta)}{\partial \theta^2}
    & = J_i^\top H_i J_i + \sum_{j=1}^{m_L} \left( \frac{\partial f_i(\theta)}{\partial \hat{y}_i} \right)_{j} \frac{\partial}{\partial \theta} (J_i)_{j},
\end{align*}
where 
$
	J_i = \frac{\partial \hat{y}_i}{\partial \theta},
    \text{ and } H_i = \frac{\partial^2 f_i(\theta)}{\partial (\hat{y}_i)^2}
$. The Gauss-Newton (GN) method (e.g., see \cite{wright1999numerical, martens2010deep}) approximates the Hessian matrix by ignoring the second term in the above expression, i.e., the GN approximation to $\frac{\partial^2 f_i(\theta)}{\partial \theta^2}$ is $J_i^T H_i J_i$. Note that $J_i \in R^{m_L \times n}$ and $H_i \in R^{m_L \times m_L}$, and hence that $H_i$ is a relatively small matrix. Finally, $\frac{\partial^2 f(\theta)}{\partial \theta^2}$ is approximated by
\begin{align}
	B_t = \frac{1}{N} \sum_{i=1}^N J_i^T H_i J_i.
    \label{equation_11}
\end{align}

\subsubsection{Natural gradient method}

The natural gradient (NG) method (\cite{amari1998natural}) modifies the gradient $\nabla f(\theta)$ by multiplying it by the inverse of the Fisher (information) matrix, which serves as an approximation to $\frac{\partial^2 f(\theta)}{\partial \theta^2}$: 
\begin{align}
    B_t = F_t \equiv \frac{1}{N} \sum_{i=1}^n \nabla f_i(\theta^{(t)}) \nabla f_i(\theta^{(t)})^\top.
    \label{equation_31}
\end{align}

\subsubsection{Properties of the approximations}

There are several reasons why the NG and GN methods are well-suited for training NNs. First, even though the loss function (\ref{equation_15}) is a non-convex function of $\theta$, $H_i$ is positive semi-definite ($H_i \succeq 0$) for commonly-used loss functions (e.g., least-squared loss, cross entropy loss). Hence, $J_i^T H_i J_i \succeq 0$. Also, $F_t \succeq 0$. $p_t$ is a descent direction, as long as $-g_t$ is a descent direction and $g_t$ is not in the null space of $B_t$.
Second, the multiplication of an arbitrary vector by the matrix $B_t$ or $F$ can be done efficiently by backpropagation (see appendix or, e.g., \cite{schraudolph2002fast}).


\subsection{Mini-batch and damping}
\label{section_damping}

The prohibitively large amount of data and (relative) difficulty in computing the GN and Fisher matrices suggests simplifying these approximations to the Hessian matrix further. Consequently, as in \cite{martens2010deep, martens2015optimizing}, we estimates (\ref{equation_11}) and (\ref{equation_31}) using a mini-batch of indices $S_2^t \subset \{ 1, 2, ..., N \}$ at iteration $t$ where $|S_2^t| = N_2$. 




Mini-batch approximations make the GN and Fisher matrices low-rank. Hence, we add $\lambda I$ to them to make them invertible (namely, the Levenberg-Marquardt (LM) method (\cite{more1977levenberg})). 
Thus, the approximation to the Hessian becomes
\begin{align}
	G_t
	= B_t + \lambda I,
	\label{equation_22}
\end{align}
where $B_t$ is either the Fisher information matrix or the Gauss-Newton matrix.

Viewing the LM method as a trust-region method, the magnitude of $\lambda$ is inversely related to the size of the region $||p|| \le \Delta_t$ in which we are confident about the ability of the quadratic model
\begin{align*}
    m_t(p) = f(\theta^{(t)}) + g_t^T p + \frac{1}{2} p^T B_t p
\end{align*}
to approximate $f(\theta^{(t)} + p)$.
Note that solving $B_t p = - g_t$ is equivalent as minimizing $m_t(p)$.

To determine the value of $\lambda$, let $\lambda = \lambda_{\text{LM}} + \tau$, where $\lambda_{\text{LM}}$ is updated at each iteration, and $\tau > 0$. $\tau$ is typically very small and can be viewed as coming from an $l_2$ regularization term in the objective function, which is a common practice in training deep NNs to avoid possible over-fitting. It also ensures that $\lambda_t \geq \tau > 0$, which guarantees that the smallest eigenvalue of $G_t$ is strictly positive.

To update $\lambda_{\text{LM}}$, we consider the ratio of the actual reduction in $f(\cdot)$ to the reduction in the quadratic model $m_t(\cdot)$
\begin{align}
    \rho_t = \frac{f(\theta^{(t)}) - f(\theta^{(t)} + p_t)}{ m_t(0) - m_t(p_t)}
    \label{equation_24}
\end{align}
to measure how "good" that model is. If $\rho_t$ is positive and large, it means that the quadratic model is a good approximation. Hence, we enlarge the "trust region", by decreasing the value of $\lambda_{\text{LM}}$. If $\rho_t$ is small, we increase the value of $\lambda_{\text{LM}}$ (see Section \ref{section_10} for more intuition). Specifically, $\lambda_{\text{LM}}$ is updated as follows: 
{\bf if} $\rho_t < \epsilon$: $\lambda_{\text{LM}}^{(t+1)} = boost \times \lambda_{\text{LM}}^{(t)}$;
{\bf else if} $\rho_t > 1 - \epsilon$: $\lambda_{\text{LM}}^{(t+1)} = drop \times \lambda_{\text{LM}}^{(t)}$;
{\bf else}: $\lambda_{\text{LM}}^{(t+1)} = \lambda_{\text{LM}}^{(t)}$, where $0 < \epsilon < \frac{1}{2}$, $drop < 1 < boost$. 
Finally, $\lambda_{t+1} = \lambda_{\text{LM}}^{(t+1)} + \tau$.










\section{Our Innovation: a general framework for computing $p_t$}
\label{section_6}

In the NN context, it is very expensive to compute (\ref{equation_11}) or (\ref{equation_31}); and even given $G_t$, computing $p_t$ still requires $O(n^3)$ time, which is prohibitive. For these reasons \cite{martens2010deep} proposed a Hessian-free method that uses an "early termination" linear conjugate gradient method to compute $p_t$ approximately. Here we propose an alternative approach, that is both potentially faster, and is also exact.


\subsection{Using the Sherman-Morrison-Woodbury (SMW) Formula}

The matrix $G_t$ for both the GN and NG methods has the form $G_t = \lambda I + \frac{1}{N_2} J^\top H J$, where $J^\top = (J_1^\top, \cdots, J_{N_2}^\top)$ and $H = \text{diag} \{ H_1, \cdots, H_{N_2} \}$ for GN and $J^\top = (\nabla f_1(\theta), \cdots, \nabla f_{N_2}(\theta))$ and $H = I$ for NG. Using the well-known SMW formula,
\begin{align}
    G_t^{-1} = \frac{1}{\lambda} \left( I - \frac{1}{N_2} J^\top D_t^{-1} J \right),
    \text{ where }
    D_t = \lambda H^{-1} + \frac{1}{N_2} J J^\top.
    \label{equation_34}
\end{align}

Note that the matrix $D_t$ in (\ref{equation_34}) is ${N_2 m_L \times N_2 m_L}$ in the GN case and ${m_L \times m_L}$ in the NG case, much smaller than the $n \times n $ LM matrix $G_t$, assuming $N_2$ is not too large in the GN case. 

In cases where the $H_i$ are not invertible (e.g., softmax regression with GN method), we can still use SMW to obtain
\begin{align}
    G_t^{-1} = \frac{1}{\lambda} \left( I - \frac{1}{N_2} J^\top H D_t^{-1} J \right)
    \text{, where }
    D_t = \lambda I + \frac{1}{N_2} J J^\top H.
    \label{equation_35}
\end{align}
Because the analysis for these cases are similar to those where $H_i$ is invertible, we will restrict our analysis to the symmetric expressions in (\ref{equation_34}).

\subsection{Backpropagation in SMW}

For an arbitrary vector $V \in R^{m_L}$,
$
    J_i^\top V = \left( \frac{\partial \hat{y}_i}{\partial \theta} \right)^\top V
    = \frac{\partial \left( (\hat{y}_i)^\top V \right)}{\partial \theta}.
$
Hence, we can compute the vector $J ^{\top}_i V$ by backpropagating through the customized function $(\hat{y}_i)^\top V$. The other vectors needed in (\ref{equation_34}) can be computed similarly (See appendix).


\subsection{Computing \texorpdfstring{$D_t$}{Lg}}
\label{section_9}

We first demonstrate how to compute $D_t$ in (\ref{equation_34}) in an efficient way. For a given data point $i$, let $D W_i^{(l)}$ denote the gradient of $f_i(\theta)$ w.r.t $W^{(l)}$. 
As shown in the appendix, $D W_i^{(l)}$ is a rank-one matrix, i.e., $D W_i^{(l)} = (g_i^{(l)}) (v_i^{(l-1)})^\top$. Hence, the $(i, j)$ element of $D_t$ can be computed as
\begin{align*}
    \nabla f_i(\theta)^\top \nabla f_i(\theta)
    & = \sum_{l=1}^L \text{vec}\left( D W_i^{(l)} \right)^\top \text{vec}\left( D W_j^{(l)} \right)
    \\ & = \sum_{l=1}^L \text{vec}\left( (g_i^{(l)}) (v_i^{(l-1)})^\top \right)^\top \text{vec}\left( (g_j^{(l)}) (v_j^{(l-1)})^\top \right)
    \\ & = \sum_{l=1}^L \left( (g_i^{(l)})^\top (g_j^{(l)}) \right) \left( (v_i^{(l-1)})^\top v_j^{(l-1)} \right)
\end{align*}
For simplicity, we have ignored the $b$'s in the above. Therefore, we compute $D_t$ without explicitly writing out any $D W_i^{(l)}$, where all the vectors needed have been computed when
doing backpropagation for the gradient. 

Similarly, in the case of the GN matrix where $D_t$ is defined in (\ref{equation_34}), we need to compute $J_{i_1} J_{i_2}^\top$ for all $i_1, i_2 = 1, ..., N$. The $(j_1, j_2)$ element of $J_{i_1} J_{i_2}^\top$, namely $e_{j_1}^\top J_{i_1} J_{i_2}^\top e_{j_2}$, is the dot product of two "backpropagated" gradients $J_{i_1}^\top e_{j_1}$ and $J_{i_2}^\top e_{j_2}$, and hence can be computed efficiently. 
\footnote{
There are other ways to compute and "invert" $D_t$, e.g.,  solving $D_t d_t = -J g_t$ by the linear conjugate gradient method as in Hessian-free, with either the explicit value of $D_t$, or an oracle to compute the product of $D_t$ with an arbitrary vector. 
We tried both of these approaches and neither performed better that inverting $D_t$, i.e., computing $d_t$ exactly.
}




\section{Algorithm for Subsampled Second-Order Methods}

In this section, we summarize our subsampled GN and NG methods. 
Since we are focused on very large data sets, we estimate the gradient $\nabla f( )$, and $f( )$ in the reduction ratio $\rho_t$ (see (\ref{equation_24})) using a mini-batch $S^t_1$.


\begin{algorithm}[h]
	\caption{Sub-sampled Gauss-Newton / Natural Gradient method}
	\label{algo_1}
	\begin{algorithmic}[1]

        \STATE {\bf Parameters:} $N_1$, $N_2$, $0 < \epsilon < \frac{1}{2}$, learning rate $\alpha$
    	
        \FOR {$t = 0, 1, 2, ...$}
            
            \STATE
            Randomly select a mini-batch $S_1^t \subseteq [N]$ of size $N_1$ and $S_2^t \subseteq S_1^t$ of size $N_2$
    			
            \STATE Compute $g_t = \frac{1}{|S_1^t|} \sum_{i \in S_1^t} \nabla f_i(\theta^{(t)})$
            
            \STATE Compute $D_t$ and $p_t = - G_t^{-1} g_t$ as in (\ref{equation_34}) or (\ref{equation_35}) with mini-batch $S_2^t$

    		

            
            \STATE Update $\lambda$ using the LM style rule $\{$ see Section \ref{section_damping} $\}$ with $S_1^t$ mini-batch estimates of $f()$ to compute $\rho_t$ in (\ref{equation_24}) 
            

            
            \STATE
            set  $\theta^{(t+1)} = \theta^{(t)} + \alpha \cdot p_t$

    	\ENDFOR
	\end{algorithmic}
\end{algorithm}

The above algorithm works for both the GN and NG methods, the only differences being in computing and inverting $D_t$ and the backpropagations needed for computing $G_t^{-1} g_t$.

\section{Convergence}
\label{section_10}

Recall that the LM direction that we compute is
$p_t = - (B_t + \lambda_t I)^{-1} g_t.$
If we let $\Delta_t = ||p_t||$, it is well known that $p_t$ is the global solution to the trust-region (TR) problem
\begin{align*}
    \min_{p} m_t(p) \text{ s.t. } ||p|| \le \Delta_t.
\end{align*}

As in the classical TR method, we evaluate the quality of the quadratic model $m_t(\cdot)$ by computing $\rho_t$ defined by (\ref{equation_24}). However, while the classical TR method updates $\Delta_t$ depending on the value of $\rho_t$, we follow the LM approach of updating $\lambda_t$ instead. Loosely speaking, there is a reciprocal-like relation between $\lambda_t$ and $\Delta_t$. While Martens \cite{martens2010deep} proposed this way of updating $\lambda_t$ as a "heuristic", we are able to show that Algorithm \ref{algo_1}, with a exact (full) gradient (i.e., $N_1 = N$) and only updating $\theta_t$ when $\rho_t$ is above a certain threshold (say $\eta$),
converges to a stationary point under the following assumptions:


\begin{assumption}
    \label{assumption_5}
    $||B_t|| \le \beta$.
\end{assumption}

\begin{assumption}
    \label{assumption_6}
    $||\nabla^2 f(\theta)|| \le \beta_1$.
\end{assumption}


Our proof is similar to that used to prove convergence of the standard trust-region method ( e.g., see \cite{wright1999numerical}), and in particular makes use of the following:
\begin{lemma}
    \label{lemma_1}
    Under Assumption \ref{assumption_5}, there exists a constant $c_1 > 0$ such that
    $$m_t(0) - m_t(p_t) \ge c_1 ||g_t|| \Delta_t.$$
\end{lemma}

\begin{proof}
Because $p_t = - (B_t + \lambda_t I)^{-1} g_t$, $-g_t^T p_t = p_t^T \left( B_t + \lambda_t I \right) p_t$. Then since $B_t \succeq 0$ , we have
\begin{align*}
    m_t(0) - m_t(p_t)
    & = - g_t^T p_t - \frac{1}{2} p_t^T B_t p_t
 \ge \lambda_t ||p_t||^2.
\end{align*}
On the other hand, since  $\frac{\lambda_t}{\beta + \lambda_t} \ge \frac{\tau}{\beta + \tau} = c_1 > 0$, $||B_t|| \le \beta$  and $\Delta_t = ||p_t||$,
\begin{align*}
    c_1 ||g_t|| \Delta_t
    & = c_1 ||- (B_t + \lambda_t I) p_t|| ||p_t||
    \le c_1 \left( \beta + \lambda_t \right) ||p_t||^2
\le \lambda_t ||p_t||^2.
\end{align*}
\end{proof}

Using Lemma \ref{lemma_1}, we now prove the global convergence of the full-gradient variant of Algorithm \ref{algo_1}:
\begin{theorem}
    \label{theorem_3}
    Suppose in Algorithm \ref{algo_1}, we set $N_1 = N$, $\alpha = 1$ and only update $\theta_t$ when $\rho_t \ge \eta$ where $0 < \eta < \epsilon$.
    Then, under Assumptions \ref{assumption_5} and \ref{assumption_6}, if $f$ is bounded below. we have that
    $\lim_{t \to \infty} ||g_t|| = 0.$
\end{theorem}

\begin{proof}

We first show that $\lambda_t$ is bounded above by some constant $\Lambda_1$: Recalling (\ref{equation_24}), at iteration $t$, we have
\begin{align}
    |\rho_t - 1|
    & = \left| \frac{m_t(p_t) - f(\theta^{(t)} + p_t)}{ m_t(0) - m_t(p_t)} \right|.
    \label{equation_25}
\end{align}
By Taylor's theorem,
$
    f(\theta^{(t)} + p_t) = f(\theta^{(t)}) + \nabla f(\theta^{(t)})^T p_t + \frac{1}{2} p_t^T \nabla^2 f(\theta^{(t)} + \mu p_t) p_t
$
for some $\mu \in (0, 1)$. Hence,
\begin{align}
    & \left| m_t(p_t) - f(\theta^{(t)} + p_t) \right|
    = \left| \frac{1}{2} p_t^T B_t p_t - \frac{1}{2} p_t^T \nabla^2 f(\theta^{(t)} + \mu p_t) p_t \right|
    \le \frac{1}{2} (\beta + \beta_1) ||p_t||^2.
    \label{equation_26}
\end{align}



By (\ref{equation_25}) and (\ref{equation_26}), we have that
\begin{align*}
    |\rho_t - 1|
    \le \frac{\frac{1}{2} (\beta + \beta_1) ||p_t||^2}{\lambda_t ||p_t||^2}
    = \frac{\frac{1}{2} (\beta + \beta_1)}{\lambda_t}.
\end{align*}
Hence, there exists a $\Lambda > 0$ such that for all $\lambda_t \ge \Lambda$, we have $|\rho_t - 1| \le \epsilon$, and thus, $\rho_t \ge 1 - \epsilon$. Consequently, by the way $\lambda_{\text{LM}}$ is updated, for all $t$, $\lambda_t \le boost \cdot                                                                      \Lambda + \tau= \Lambda_1$; i.e.,
 $\lambda_t$ is bounded above. By Assumption \ref{assumption_5}$, ||B_t + \lambda_t I||$ is also bounded, i.e.,  
$
    ||B_t + \lambda_t I||
    \le \beta + \Lambda_1
$.
Hence, the minimum eigenvalue of $(B_t + \lambda_t)^{-1}$ is no less than $\frac{1}{\beta + \Lambda_1}$. Finally, by the Cauchy-Schwarz inequality and the fact that $\Delta_t = ||p_t|| = ||-(B_t + \lambda_t I)^{-1} g_t||$, we have that
$
    ||g_t|| \Delta_t
    \ge ||g_t (B_t + \lambda_t I)^{-1} g_t||
    \ge \frac{1}{\beta + \Lambda_1} ||g_t||^2
$.

Let $T_1 = \{ t = 0,1,... \ | \ \rho_t \ge \eta \}$ denote the set of indices $t$ such that step $p_t$ is accepted. For any $t \in T_1$, by definition of $\rho_t$ and Lemma \ref{lemma_1},
\begin{align}
    f(\theta^{(t)}) - f(\theta^{(t+1)})
    > \eta c_1 ||g_t|| \Delta_t
    \ge \eta c_1  \frac{1}{\beta + \Lambda_1} ||g_t||^2.
    \label{equation_30}
\end{align}

We now show that $|T_1| = \infty$ (unless for some $t$, $g_t = 0$ and Algorithm \ref{algo_1} stops finitely): Suppose
that this is not the case. Then there exists a $T > 0$ such that for all $t \ge T$, $p_t$ is rejected (i.e., $\rho_t \le \eta < \epsilon$). Then, $\lambda_t \to \infty$, contradicting the fact that $\lambda_t$ is bounded. Because $|T_1| = \infty$, $\lim_{t \to \infty} ||g_t|| = \lim_{t \in T_1} ||g_t||$. Because $f$ is bounded below and $f(\theta^{(t)})$ is non-increasing, the left-hand-side of (\ref{equation_30}) goes to zero. Hence, the right-hand-side also goes to zero, which implies $\lim_{t \in T_1} ||g_t|| = 0$. 
\end{proof}

\section{Computational Costs of Proposed Algorithms}

In this section we discuss the computational cost of our SMW-based GN and NG algorithms, and compare them with SGD, Hessian-free (HF) and KFAC. 
First, several basic operations including computing $f_i(\theta)$, $\nabla f_i(\theta)$, $J V$ and $J^T v$ all requires $O(n)$ time for a single data point. Hence, all algorithms have a cost of $O(N_1 n)$ for computing the stochastic gradient $g_t$. 

For the second order methods, the following table summarizes the extra costs for computing the LM direction $p_t$, where $n_{\text{HF}}$ denotes the number of CG iterations used in Hessian-free.

\begin{table}[h]
\centering
\vskip 0.1in

\setlength{\extrarowheight}{4.5 pt}

\begin{tabular}{ccccccc}
    \hline
    Algorithm & Cost
    \\ 
    \hline
    SMW-GN &  $O(m_L N_2 n + m_L^2 N_2^2 \sum_l m_l + m_L^3 N_2^3)$
    \\
    SMW-NG & $O(N_2 n + N_2^2 \sum_l m_l + N_2^3)$
    \\
    HF & $n_{\text{HF}} \times O(N_2 n)$
    \\
    KFAC & $O(\sum_l m_l^3 + N_2 \sum_l m_l^2)$
	\\
	\hline

\end{tabular}
\end{table}

\subsection{Comparison Between Algorithms}



Since $n$ is usually extremely large in NNs, we see that in SMW-GN the multiplier of the term involving $n$ is reduced from $n_{\text{HF}} N_2$  in HF to $m_L N_2$. KFAC has a term proportional to $\sum_l m_l^3$, which is of an even higher order than $n$. 


For all of the second-order methods, when $N_2 \ll N_1$, the overhead for each iteration is usually compensated for by the better direction generated by these methods for updating the parameters. However, even if the condition $N_2 \ll N_1$ is not met, as long as $N_2$ is reasonably small, the overhead is controllable. Consequently, one should choose a relatively small $N_2$ when implementing our SMW-based algorithms. 


\section{Numerical Experiments}

\label{section_17}

\begin{figure}[h]
    \hspace*{-1.6cm}\includegraphics[width = 1.25 \textwidth]{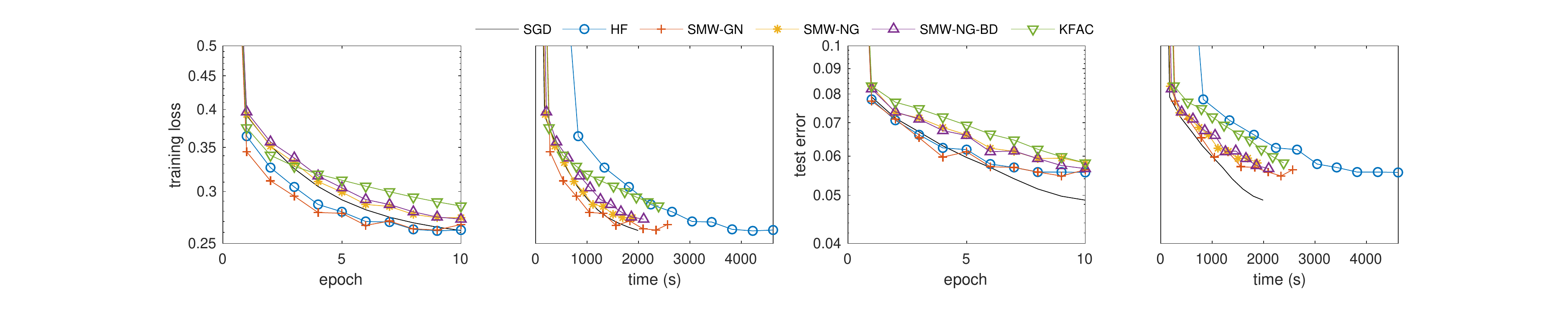}
    \caption{Results on MNIST classification problem with $N_1 = 60$, $N_2 = 30$. Learning rates: 0.1}
    \label{figure_6}
\end{figure}

\begin{figure}[h]
    \hspace*{-1.6cm}\includegraphics[width= 1.25 \textwidth]{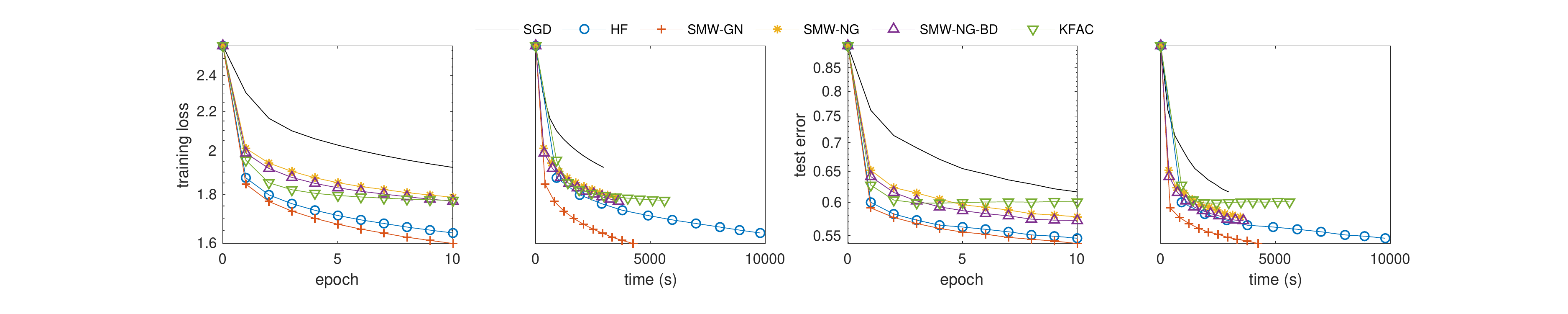}
    \caption{Results on CIFAR classification problem with $N_1 = 100$, $N_2 = 50$. Learning rates: 0.01}
    \label{figure_7}
\end{figure}

\begin{figure}[h]
    \hspace*{-1.6cm}\includegraphics[width= 1.25 \textwidth]{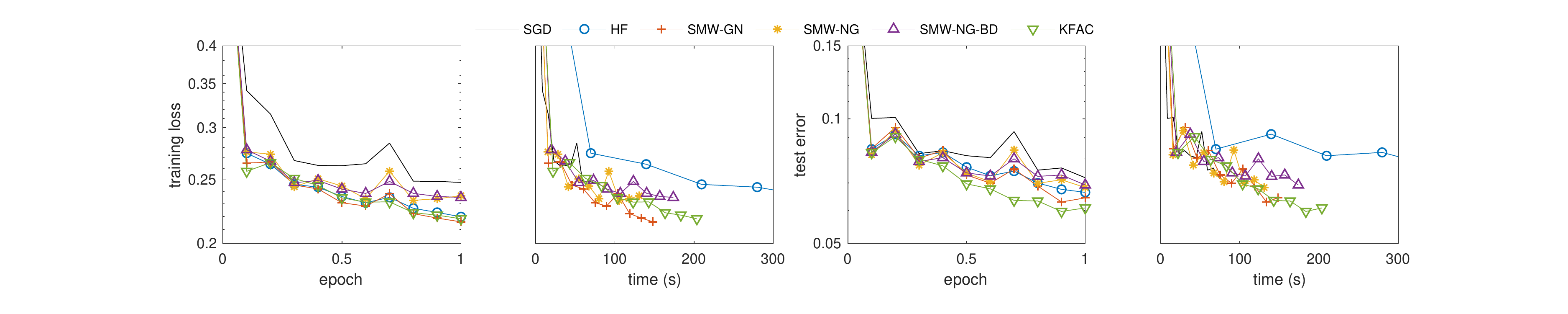}
    \caption{Results on webspam classification problem with $N_1 = 60$, $N_2 = 30$. Learning rates: 0.4 (SGD), 0.05 (HF and SMW-GN), 0.1 (SMW-NG, SMW-NG and KFAC)}
    \label{figure_8}
\end{figure}

We compared our algorithms SMW-GN and SMW-NG with SGD, HF and KFAC. The KFAC algorithm was implemented using block-diagonal approximaton, without re-scaling and momentum (see section 6.4 and 7 of \cite{martens2015optimizing}). We also included a block diagonal version of SMW-NG, namely, SMW-NG-BD, where in each block, the matrix is inverted by SMW, in order to mimic the block diagonal approximation used by KFAC. 

For all of the experiments reported in this section, we set the initial value of $\lambda_{\text{LM}}$ to be 1, $boost = 1.01$, $drop = 0.99$, $\epsilon = 1/4$, $\tau = 0.001$, same as in \cite{kiros2013training}.
All algorithms were implemented in MATLAB R2019a and run on an Intel Core i5 processor.
We tested the performance of the algorithms on several classification problems. We reported both training loss and testing error. The data sets were scaled to have zero means and unit variances.


    {\bf MNIST} (\cite{lecun1998gradient}):
    The training set is of size $N = 6 \times 10^4$. We used a NN with one hidden layer of size 500 and logistic activation, i.e., $(m_0, m_1, m_2) = (784, 500, 10)$, where the first and last layers are the size of input and output. The output layer was softmax with cross entropy.  The learning rate for SGD was set to be 0.1, tuned from $\{ 0.01, 0.05, 0.1, 0.5 \}$. Initial learning rates for other methods were also set to be 0.1 for purposes of comparison. We did not tune the learning rates for second-order methods because they adaptively modify the rate by updating $\lambda_{\text{LM}}$ as they proceed. We ran each algorithm for 10 epochs. The results are presented in Figure \ref{figure_6}.



    
    {\bf CIFAR-10} (\cite{krizhevsky2009learning}):
    The training set is of size $N = 5 \times 10^4$. We used a NN with two hidden layers of size 400 and logistic activation, i.e., $(m_0, m_1, m_2, m_3) = (3072, 400, 400, 10)$. The output layer was softmax with cross entropy. The learning rate was set to be 0.01. We ran each algorithm for 10 epochs. The results are presented in Figure \ref{figure_7}.

    {\bf webspam} (\cite{CC01a}):
    The training set is of size $N = 3 \times 10^5$. We used a NN with two hidden layers of size 400 and logistic activation, i.e., $(m_0, m_1, m_2, m_3) = (254, 400, 400, 1)$. Because this is a binary classification, we set the output layer to be logistic with binary cross entropy. 
    We tuned the learning rate for SGD and initial learning rates for other algorithms separately, and the results shown in Figure \ref{figure_8} all used their corresponding best learning rates, which are indicated there.
    We ran each algorithm for 1 epoch. 

\subsection{Discussion of results}


Interestingly, the relative ranking of the algorithms changes from one problem to another, indicating that the relative performance of the algorithms depends upon the data set, structure of the NN, and parameter settings. 

From our experimental results, we see that SMW-GN is always faster than HF in terms of both epochs and clock-time, which is consistent with our analysis above. 
KFAC sometimes performs very well, not surprisingly, because it accumulates more and more curvature information with each new mini-batch. But it also slows down considerably when the NN has wide layers (see Figure \ref{figure_7}).  
Moreover, the three experiments were done differently, mimicing the different practices used when training a NN model, namely, tuning learning rates for all algorithms, tuning learning rates for one algorithm and then using it for all, or simply choosing a conservative learning rate.
After carefully tuning the learning rate, SGD can perform as well as second-order methods as shown in Figure \ref{figure_6}. However, if learning rate is chosen to be more conservative or typical (e.g., 0.01), it may suffer from slow convergence compared with second-order methods (see Figure \ref{figure_7}). If we want to get lower training loss or testing error, we may have to run it for far more epochs / time.

The key take-away from our numerical results is that our SMW-based algorithms that are based on the Gauss-Newton and natural 
gradient methods are competitive with their Hessian-free and Kronecker-factor implementations, HF and KFAC, as well as SGD. In particular, SMW-GN performs extremely well without requiring any parameter tuning.

\section{Summary and Future Research Directions}

In this paper, we proposed efficient LM-NG/GN methods for training neural networks, semi-stochastic versions of which are provably convergent, while fully stochastic versions are competitive with off-the-shelf algorithms including SGD and KFAC. A promising future research topic is the study of how to adapt gradient (diagonal) rescaling techniques like Adam \cite{kingma2014adam} and AdaGrad \cite{duchi2011adaptive}, that are based on running averages of the first and second moments of the stochastic gradients encountered during the course of the algorithm, to our GN and NG algorithms.  This is a challenging topic, since both NG and NG based algorithms already incorporate non-diagonal rescalings. A second promising future research topic is the study of how to, starting from relatively small gradient and GN and Fisher matrix mini-batches, increase their sizes as needed, by evaluating their variances (e.g., see \cite{bollapragada2018adaptive} and references therein).  Finally,  the structure of our algorithms is well-suited for parallel computation. Besides the common approach of distributing the data across different processors, one can compute terms involving $D_t$ in parallel, so that the cost of second-order computations becomes comparable to that for evaluating gradients.    




\medskip
\small
\bibliography{neurips_2019.bib}

\newpage



\appendix

\section{Computational techniques}
\label{section_16}

In this section, we present the major computational techniques used by our algorithm, and present their pseudo-codes.

\subsection{Network computation (forward pass)}

We have the Algorithm \ref{algo_6}. 

\begin{algorithm}[h]
	\caption{Forward Pass: Compute the neural network w.r.t. a single input $x$
	}
	\label{algo_6}
	\begin{algorithmic}[1]
	\STATE {\bf Input:} $\theta$, $x$
    \STATE {\bf Output:} $\hat{y}$ or $h^{(l)}$, $v^{(l)}$ ($l = 1, ..., L$)

    \STATE unpack $\theta$ to be $W^{(l)}$, $b^{(l)}$ ($l = 1, ..., L$)

	\STATE $v^{(0)} = x$
    \FOR {$l = 1, ..., L$}
    	\STATE $h^{(l)} = W^{(l)} v^{(l-1)} + b^{(l)}$
    	\STATE $v^{(l)} = \phi^{(l)}(h^{(l)})$
    \ENDFOR
    \STATE $\hat{y} = v^{(L)}$

	\end{algorithmic}
\end{algorithm}

\subsection{Gradient computation (backward pass)}
\label{section_15}

In order to compute the gradient $\nabla f(\theta)$, it suffices to compute $\nabla f_i(\theta)$ for $i = 1, ..., N$.

For $i = 1, ..., N$,
\begin{align*}
	& \nabla f_i(\theta)
    = \frac{\partial f_i(\theta)}{\partial \theta}
    = \frac{\partial \varepsilon(\hat{y}_i(\theta), y_i)}{\partial \theta}
    = \frac{\partial \varepsilon(\hat{y}_i(\theta), y_i)}{\partial \hat{y}_i(\theta)} \frac{\partial \hat{y}_i(\theta)}{\partial \theta}
    = \frac{\partial \varepsilon(\hat{y}_i, y_i)}{\partial \hat{y}_i} \frac{\partial v_i^{(L)}}{\partial \theta}.
\end{align*}
Hence,
\begin{align}
    & \frac{\partial f_i(\theta)}{\partial b^{(l)}}
    = \frac{\partial \varepsilon(\hat{y}_i, y_i)}{\partial \hat{y}_i} \frac{\partial v_i^{(L)}}{\partial b^{(l)}}
    \nonumber
    = \frac{\partial \varepsilon(\hat{y}_i, y_i)}{\partial \hat{y}_i}
    \frac{\partial v_i^{(L)}}{\partial h_i^{(L)}}
    \frac{\partial h_i^{(L)}}{\partial v_i^{(L-1)}}
    \cdots
    \frac{\partial h_i^{(l+1)}}{\partial v_i^{(l)}}
    \frac{\partial v_i^{(l)}}{\partial h_i^{(l)}}
    \frac{\partial h_i^{(l)}}{\partial b^{(l)}}
    \nonumber
    \\ = &
    \frac{\partial \varepsilon(\hat{y}_i, y_i)}{\partial \hat{y}_i}
    \frac{\partial v_i^{(L)} }{\partial h_i^{(L)}}
    W^{(L)}
    \cdots
    W^{(l+1)}
    \frac{\partial v_i^{(l)} }{\partial h_i^{(l)}}.
    \label{equation_12}
\end{align}

Since
\begin{align*}
	& \frac{\partial h_i^{(l)}}{\partial \text{vec}({W}^{(l)})}
	= 
    \begin{pmatrix}
    	\frac{\partial h_i^{(l)}}{\partial W_{:,1}^{(l)}} \cdots \frac{\partial h_i^{(l)}}{\partial W_{:,m_{l-1}}^{(l)}}
    \end{pmatrix}
    =
    \begin{pmatrix}
    	(v_i^{(l-1)})_1 I_{m_l \times m_l} & \cdots & (v_i^{(l-1)})_{m_{l-1}} I_{m_l \times m_l}
    \end{pmatrix},
\end{align*}
similarly,
\begin{align*}
    \frac{\partial f_i(\theta)}{\partial \text{vec}({W}^{(l)})}
    = & \frac{\partial \varepsilon(\hat{y}_i, y_i)}{\partial \hat{y}_i}
    \frac{\partial v_i^{(L)}}{\partial h_i^{(L)}}
    \frac{\partial h_i^{(L)}}{\partial v_i^{(L-1)}}
    \cdots
    \frac{\partial h_i^{(l+1)}}{\partial v_i^{(l)}}
    \frac{\partial v_i^{(l)}}{\partial h_i^{(l)}}
    \frac{\partial h_i^{(l)}}{\partial \text{vec}({W}^{(l)})}
    \\ = &
    \frac{\partial \varepsilon(\hat{y}_i, y_i)}{\partial \hat{y}_i}
    \frac{\partial v_i^{(L)} }{\partial h_i^{(L)}}
    W^{(L)}
    \cdots
    W^{(l+1)}
    \frac{\partial v_i^{(l)} }{\partial h_i^{(l)}}
    \\ & \cdot
    \begin{pmatrix}
    	(v_i^{(l-1)})_1 I_{m_l \times m_l} & \cdots & (v_i^{(l-1)})_{m_{l-1}} I_{m_l \times m_l}
    \end{pmatrix}.
\end{align*}
Thus,
\begin{align}
    \frac{\partial f_i(\theta)}{\partial W^{(l)}}
    =
    \frac{\partial \varepsilon(\hat{y}_i, y_i)}{\partial \hat{y}_i}
    \frac{\partial v_i^{(L)} }{\partial h_i^{(L)}}
    W^{(L)}
    \cdots
    W^{(l+1)}
    \frac{\partial v_i^{(l)} }{\partial h_i^{(l)}}
    (v_i^{(l-1)})^T.
    \label{equation_13}
\end{align}

By (\ref{equation_12}) and (\ref{equation_13}), we have the following recursion
\begin{align*}
	\frac{\partial f_i(\theta)}{\partial b^{(l)}}
    & = \frac{\partial f_i(\theta)}{\partial b^{(l+1)}} W^{(l+1)}
    \frac{\partial v_i^{(l)} }{\partial h_i^{(l)}},
    \\ \frac{\partial f_i(\theta)}{\partial W^{(l)}}
    & = \frac{\partial f_i(\theta)}{\partial b^{(l)}} (v_i^{(l-1)})^T.
\end{align*}

Combining all of the above yields Algorithm \ref{algo_7}. 

\begin{algorithm}[h]
	\caption{Backward Pass: Compute the gradient $\nabla f_i(\theta)$
	}
  	\label{algo_7}
	\begin{algorithmic}[1]
    	\STATE {\bf Input:} $\theta$, $h^{(l)}$, $v^{(l)}$ ($l = 1, ..., L$), $x$, $y$
    	\STATE {\bf Output:} $\nabla f_i(\theta)$

        \STATE unpack $\theta$ to be $W^{(l)}$, $b^{(l)}$ ($l = 1, ..., L$)

        
        \STATE $g^{(L)} = \frac{\partial \varepsilon(v^{(L)}, y)}{\partial \hat{y}}
        \frac{\partial v_i^{(L)} }{\partial h_i^{(L)}}$
        \STATE $b_{1}^{(L)} = g^{(L)}$
    	\STATE $W_{1}^{(L)} = g^{(L)} (v^{(L-1)})^T$
        \FOR {$l = L-1, ..., 1$}
        	\STATE $g^{(l)} = g^{(l+1)} W^{(l+1)}
            \frac{\partial v_i^{(l)} }{\partial h_i^{(l)}}$
            \STATE $b_{1}^{(l)} = g^{(l)}$
            \STATE $W_{1}^{(l)} = g^{(l)} (v^{(l-1)})^T$
        \ENDFOR

    	\STATE pack $W_1^{(l)}$, $b_1^{(l)}$ ($l = 1, ..., L$) to be $\nabla f_i(\theta)$
        \RETURN $\nabla f_i(\theta)$
	\end{algorithmic}
\end{algorithm}

\subsection{\texorpdfstring{$J_i$}{Lg}}

Although we do not explicitly compute $J_i$ in our algorithms, deriving an expression for $J_i$ will help us in deriving expressions for the quantities we need.


Noticing that $J_i = \frac{\partial \hat{y}_i}{\partial \theta} = \frac{\partial v_i^{(L)}}{\partial \theta}$, we'd like to get an recursion w.r.t. $\frac{\partial v_i^{(0)}}{\partial \theta}$, ..., $\frac{\partial v_i^{(L)}}{\partial \theta}$. Because $v_i^{(0)} \equiv x_i$, we have that $\frac{\partial v_i^{(0)}}{\partial \theta} = 0$. For $l = 1, ..., L$,
\begin{align}
	\frac{\partial h_i^{(l)}}{\partial \theta}
    & = \frac{\partial \left( W^{(l)} v_i^{(l-1)} + b^{(l)} \right)}{\partial \theta}
    \nonumber
    = \frac{\partial W^{(l)}}{\partial \theta} v_i^{(l-1)} + W^{(l)} \frac{\partial v_i^{(l-1)}}{\partial \theta} + \frac{\partial b^{(l)}}{\partial \theta},
    \\ \frac{\partial v_i^{(l)}}{\partial \theta}
    & = \frac{\partial v_i^{(l)}}{\partial h_i^{(l)}} \frac{\partial h_i^{(l)}}{\partial \theta},
    \label{equation_3}
\end{align}
where $\frac{\partial W^{(l)}}{\partial \theta}$, $\frac{\partial b^{(l)}}{\partial \theta}$ are some abstract notions that will be specified later.

\subsection{\texorpdfstring{$J_i \theta_1$}{Lg}}
\label{Section_5}




    We use the subscript $1$ to denote the directional derivative of some variables as a function of $\theta$ along the direction $\theta_1$. Because $J_i = \frac{\partial v_i^{(L)}}{\partial \theta}$, we have that
    \begin{align*}
        J_i \theta_1 & = \frac{\partial v_i^{(L)}}{\partial \theta} \theta_1 = v_{i,1}^{(L)}.
    \end{align*}

    We can also decompose $\theta_1$ into $\text{vec}({W}_1^{(l)})$ (hence, $W_1^{(l)}$) and $b_1^{(l)}$ (for all $l = 1, ..., L$), which agrees with the directional derivative notation.

    Note that $v_{i,1}^{(0)} = \frac{\partial v_i^{(0)}}{\partial \theta} \theta_1 = 0$. Then, recursively, by (\ref{equation_3}), for $l = 1, ..., L$,
	\begin{align}
		h_{i,1}^{(l)}
		& = \frac{\partial h_i^{(l)}}{\partial \theta} \theta_1
        \nonumber
        = \left( \frac{\partial W^{(l)}}{\partial \theta} v_i^{(l-1)} + W^{(l)} \frac{\partial v_i^{(l-1)}}{\partial \theta} + \frac{\partial b^{(l)}}{\partial \theta} \right) \theta_1
        \nonumber
        \\ & = \frac{\partial W^{(l)}}{\partial \theta} \theta_1 v_i^{(l-1)} + W^{(l)} \frac{\partial v_i^{(l-1)}}{\partial \theta} \theta_1 + \frac{\partial b^{(l)}}{\partial \theta} \theta_1
        \nonumber
        = W_1^{(l)} v_i^{(l-1)} + W^{(l)} v_{i,1}^{(l-1)} + b_1^{(l)},
        \\ v_{i,1}^{(l)}
        & = \frac{\partial v_i^{(l)}}{\partial \theta} \theta_1
        = \frac{\partial v_i^{(l)}}{\partial h_i^{(l)}} \frac{\partial h_i^{(l)}}{\partial \theta} \theta_1
        = \frac{\partial v_i^{(l)}}{\partial h_i^{(l)}}  h_{i,1}^{(l)}.
        \label{equation_4}
	\end{align}

    This leads to Algorithm \ref{algo_8}.

\begin{algorithm}[h]
	\caption{Compute the product of $J_i$ and a vector $\theta_1$
	}
	\label{algo_8}
	\begin{algorithmic}[1]
        \STATE {\bf Input:} $\theta_1$, $\theta$, $h$, $v$
        \STATE {\bf Output:} $J_i \theta_1$
        \STATE unpack $\theta$ to be $W^{(l)}$, $b^{(l)}$ ($l = 1, ..., L$)
        \STATE unpack $\theta_1$ to be $W_1^{(l)}$, $b_1^{(l)}$ ($l = 1, ..., L$)
        \STATE $v_{1}^{(0)} = 0$
        \FOR {$l = 1, ..., L$}
        	\STATE $h_{1}^{(l)} = W_1^{(l)} v^{(l-1)} + W^{(l)} v_{1}^{(l-1)} + b_1^{(l)}$
        	\STATE $v_{1}^{(l)} = \frac{\partial v^{(l)}}{\partial h^{(l)}}  h_{1}^{(l)}$
        \ENDFOR

        \RETURN $v_{1}^{(L)}$
    \end{algorithmic}
\end{algorithm}


\subsection{\texorpdfstring{$J_i^T x$}{Lg}}
\label{section_2}


The idea behind computing $J_i^T x$ ($x$ being an arbitrary vector) is even more tricky than $J_i \theta_1$. For given $J_i$ and $x$, we define $s(\theta_2) = \theta_2^T (J_i^T x) = (J_i \theta_2)^T x = (v_{i,2}^{(L)})^T x$. 

We denote the transpose of the partial derivative of $s$ w.r.t. a variable by adding a hat on the variable, e.g, $\hat{\theta}_2 = \left( \frac{\partial s}{\partial \theta_2} \right)^T$. Because $\hat{\theta}_2 = J_i^T x$, it suffices to compute
$
	\hat{\theta}_2
    = (\text{vec} \left( \hat{W}_2^{(1)} \right)^T, \left( \hat{b}_2^{(1)} \right)^T, \cdots, \text{vec} \left( \hat{W}_2^{(L)} \right)^T, \left( \hat{b}_2^{(L)} \right)^T)^T.
$

Notice that $\hat{v}_{i,2}^{(L)} = x$, which is given. For $l = L, L-1, ..., 2$, when $\hat{v}_{i,2}^{(l)}$ is given, by (\ref{equation_4}), we have that
\begin{align*}
	& v_{i,2}^{(l)}
    = \frac{\partial v_i^{(l)}}{\partial h_i^{(l)}}  h_{i,2}^{(l)}
    \\ \Rightarrow &
    \hat{h}_{i,2}^{(l)}
    = \left( \frac{\partial s}{\partial h_{i,2}^{(l)}} \right)^T
    = \left( \frac{\partial s}{\partial v_{i,2}^{(l)}} \frac{\partial v_{i,2}^{(l)}}{\partial h_{i,2}^{(l)}} \right)^T
    = \left( \frac{\partial v_{i,2}^{(l)}}{\partial h_{i,2}^{(l)}} \right)^T {\hat{v}_{i,2}^{(l)}}
    \\
    & h_{i,2}^{(l)}
	= W_2^{(l)} v_i^{(l-1)} + W^{(l)} v_{i,2}^{(l-1)} + b_2^{(l)}
    \\ \Rightarrow &
    \hat{v}_{i,2}^{(l-1)}
    = \left( \frac{\partial s}{\partial v_{i,2}^{(l-1)}} \right)^T
    = \left( \frac{\partial s}{\partial h_{i,2}^{(l)}} \frac{\partial h_{i,2}^{(l)}}{\partial v_{i,2}^{(l-1)}} \right)^T
    = \left( W^{(l)} \right)^T {\hat{h}_{i,2}^{(l)}}
    \end{align*}
    \begin{align*}
    \\ &
    \hat{W}_2^{(l)}
    = \text{vec}^{-1} \left( \widehat{\text{vec} \left( W_2^{(l)} \right)} \right)
    = \text{vec}^{-1} \left( \left( \frac{\partial s}{\partial \text{vec} \left( W_2^{(l)} \right)} \right)^T \right)
    = \text{vec}^{-1} \left( \left( \frac{\partial s}{\partial h_{i,2}^{(l)}} \frac{\partial h_{i,2}^{(l)}}{\partial \text{vec} \left( W_2^{(l)} \right)} \right)^T \right)
    \\ & \ \ \ \ \ \  = \text{vec}^{-1} \left( \left( \frac{\partial h_{i,2}^{(l)}}{\partial \text{vec} \left( W_2^{(l)} \right)} \right)^T \left( \frac{\partial s}{\partial h_{i,2}^{(l)}} \right)^T \right)
    = \text{vec}^{-1} \left( \left( \left( v_i^{(l-1)} \right)^T \otimes I_{m_l \times m_l} \right)^T {\hat{h}_{i,2}^{(l)}} \right)
    \\ & \ \ \ \ \ \  = \text{vec}^{-1} \left( \left( v_i^{(l-1)} \otimes I_{m_l \times m_l} \right) {\hat{h}_{i,2}^{(l)}} \right)
    = \hat{h}_{i,2}^{(l)} \left( v_i^{(l-1)} \right)^T
    \\ &
    \hat{b}_2^{(l)}
    = \left( \frac{\partial s}{\partial b_2^{(l)}} \right)^T
    = \left( \frac{\partial s}{\partial h_{i,2}^{(l)}} \frac{\partial h_{i,2}^{(l)}}{\partial b_2^{(l)}} \right)^T
    = {\hat{h}_{i,2}^{(l)}},
\end{align*}
where $\text{vec}^{-1}()$ is the inverse map of the "vectorization" map $\text{vec}()$.

Then, we have Algorithm \ref{algo_10}.

\begin{algorithm}[h]
	\caption{Compute the product of $J_i^T$ and a vector $x$
	}
	\label{algo_10}
	\begin{algorithmic}[1]
        \STATE {\bf Input:} $x$, $\theta$, $h$, $v$
        \STATE {\bf Output:} $J_i^T x$ or $\hat{h}_{2}^{(l)}\ (l = 1, ..., L)$

        \STATE unpack $\theta$ to be $W^{(l)}$, $b^{(l)}$ ($l = 1, ..., L$)

        \STATE $\hat{v}_{2}^{(L)} = x$

        \FOR {$l = L, ..., 1$}
            \STATE $\hat{h}_{2}^{(l)} = \left( \frac{\partial v^{(l)}}{\partial h^{(l)}} \right)^T \hat{v}_{2}^{(l)}$
        	\STATE $\hat{v}_{2}^{(l-1)} = (W^{(l)})^T \hat{h}_{2}^{(l)}$
        	\STATE $\hat{W}_2^{(l)} = \hat{h}_{2}^{(l)} (v^{(l-1)})^T$
        	\STATE $\hat{b}_2^{(l)} = \hat{h}_{2}^{(l)}$
        \ENDFOR
        \STATE pack $\hat{W}_2^{(l)}$, $\hat{b}_2^{(l)}$ ($l = 1, ..., L$) to be $J_i^T x$
    \end{algorithmic}
\end{algorithm}

Note that we have an option of outputting $J_i^T x$ or $\hat{h}_{2}^{(l)}\ (l = 1, ..., L)$. In the latter case (partial-computing mode), some operations can be skipped to save time.

\subsubsection{\texorpdfstring{Computing $J_i^T V$}{Lg}}
\label{section_8}

We present the algorithm for computing $J_i^T V$, where $V \in R^{m_L}$ is an arbitrary vector whose dimension matches the column dimension of $J_i^T$.
The vector $J_i^T V$ is of length $n$, which corresponds to the parameters $\theta$ of the neural network. We use $\hat{W}_2^{(l)}$ and $\hat{b}_2^{(l)}$ to denote the part in $J_i^T V$ corresponding to the part $W^{(l)}$ and $b^{(l)}$ in $\theta$, for $l = 1, ..., L$.

\begin{algorithm}[h]
	\caption{Compute $J_i^T V$ by backpropagation}
	\label{algo_12}
	\begin{algorithmic}[1]


        \STATE $\hat{v}_{2}^{(L)} = V$

        \FOR {$l = L, ..., 1$}
            \STATE $\hat{h}_{2}^{(l)} = \left( \frac{\partial v_i^{(l)}}{\partial h_i^{(l)}} \right)^T \hat{v}_{2}^{(l)}$
        	\STATE $\hat{v}_{2}^{(l-1)} = (W^{(l)})^T \hat{h}_{2}^{(l)}$
        	\STATE $\hat{W}_2^{(l)} = \hat{h}_{2}^{(l)} (v_i^{(l-1)})^T$
        	\STATE $\hat{b}_2^{(l)} = \hat{h}_{2}^{(l)}$
        \ENDFOR
    \end{algorithmic}
\end{algorithm}

We can compute $J_i^T V$ by a backpropagation, described in Algorithm \ref{algo_12} in $O(n)$ time. From Algorithm \ref{algo_12}, it is clear that the part of $J_i^T V$ that corresponds to a $W^{(l)}$ is the outer product of two vectors, which can be expressed as the Kronecker product of a column vector with a row vector. This observation was also made in \cite{martens2015optimizing} and \cite{botev2017practical} and can be useful when we compute $J_{i_1} J_{i_2}^T$, as shown in Section \ref{section_9}.

\subsection{\texorpdfstring{$(J_{i_1}^T x_1)^T J_{i_2}^T x_2$}{Lg}} \label{section_1}

The straightforward way to form $(J_{i_1}^T x_1)^T J_{i_2}^T x_2$ is to compute both $J_{i_1}^T x_1$ and $J_{i_2}^T x_2$ using Algorithm \ref{algo_10}, and then compute their dot product. We now present a much more efficient way to do this. In the following, we use superscripts to distinguish variables associated with $J_{i_1}^T x_1$ and $J_{i_2}^T x_2$. Since $\text{vec}(a_1 b_1^T)^T \text{vec}(a_2 b_2^T) = (b_1 \otimes a_1)^T (b_2 \otimes a_2) = (b_1^T \otimes a_1^T) (b_2 \otimes a_2) = (b_1^T b_2) \otimes (a_1^T a_2) = (b_1^T b_2) (a_1^T a_2)$, we have that
\begin{align*}
	& (J_{i_1}^T x_1)^T J_{i_2}^T x_2
    = (\hat{\theta}^{(1)}_1)^T \hat{\theta}^{(2)}_2
    = \sum_{l = 1}^L \left[ \left( \text{vec} \left( \hat{W}_1^{(l), (1)} \right) \right)^T \text{vec} \left( \hat{W}_2^{(l), (2)} \right) + (\hat{b}_1^{(l), (1)})^T \hat{b}_2^{(l), (2)} \right]
    \\ = & \sum_{l = 1}^L \left[ \left( \text{vec} \left( \hat{h}_{i_1, 2}^{(l)} (v_{i_1}^{(l-1)})^T \right) \right)^T \text{vec} \left( \hat{h}_{i_2, 2}^{(l)} (v_{i_2}^{(l-1)})^T \right) + (\hat{h}_{i_1, 2}^{(l)})^T \hat{h}_{i_2, 2}^{(l)} \right]
    \\ = & \sum_{l = 1}^L \left( (v_{i_1}^{(l-1)})^T v_{i_2}^{(l-1)} + 1 \right) \cdot \left( (\hat{h}_{i_1, 2}^{(l)})^T \hat{h}_{i_2, 2}^{(l)} \right).
\end{align*}

Hence, we can compute $(J_{i_1}^T x_1)^T J_{i_2}^T x_2$ without actually forming these two vectors. On the contrary, we can simply use the vectors $\hat{h}_{i, 2}^{(l)}$ and $v_{i}^{(l)}$ (defined in Section \ref{section_2}).

\subsection{\texorpdfstring{$
\left(
J_{i_1} J_{i_2}^T
\right)_{i_1, i_2 = 1, ..., N}
$
}{Lg}}
\label{section_4}

First, consider computing a single matrix $J_{i_1} J_{i_2}^T$ for $i_1, i_2 = 1, ..., N$. If we denote $V_{i_1, i_2}^{(l-1)} = (v_{i_1}^{(l-1)})^T v_{i_2}^{(l-1)} + 1 $, the $(j_1, j_2)$-th element of it is computed as
\begin{align*}
    & e_{j_1}^T J_{i_1} J_{i_2}^T e_{j_2}
    = (J_{i_1}^T e_{j_1})^T J_{i_2}^T e_{j_2}
    = \sum_{l = 1}^L \left( (v_{i_1}^{(l-1)})^T v_{i_2}^{(l-1)} + 1 \right) \left( (\hat{h}_{i_1, 2}^{(l), (j_1)})^T \hat{h}_{i_2, 2}^{(l), (j_2)} \right)
    \\ = & \sum_{l = 1}^L V_{i_1, i_2}^{(l-1)} \left( (\hat{h}_{i_1, 2}^{(l), (j_1)})^T \hat{h}_{i_2, 2}^{(l), (j_2)} \right).
\end{align*}
Furthermore, if we denote $\hat{H}_{i, 2}^{(l)} =
    \begin{pmatrix}
    	\hat{h}_{i, 2}^{(l), (1)} & \cdots & \hat{h}_{i, 2}^{(l), (m_L)}
    \end{pmatrix}
$, we have that $J_{i_1}^T J_{i_2} = \sum_{l = 1}^L V_{i_1, i_2}^{(l-1)} (\hat{H}_{i_1, 2}^{(l)})^T \hat{H}_{i_2, 2}^{(l)}$.

Furthermore, when computing $B$, we can use the following shortcut:
\begin{align*}
	&
	\begin{pmatrix}
    	J_1 J_1^T & \cdots & J_1 J_N^T
        \\ \cdots & \cdots & \cdots
        \\ J_N J_1^T & \cdots & J_N J_N^T
    \end{pmatrix}
    = \left(
    \begin{array}{ccc}
    	J_{i_1} J_{i_2}^T
    \end{array}
    \right)_{i_1, i_2 = 1, ..., N}
    = \left(
    \begin{array}{ccc}
    	\sum_{l = 1}^L V_{i_1, i_2}^{(l-1)} (\hat{H}_{i_1, 2}^{(l)})^T \hat{H}_{i_2, 2}^{(l)}
    \end{array}
    \right)_{i_1, i_2 = 1, ..., N}
    \\ = & \sum_{l = 1}^L \left(
    \begin{array}{ccc}
    	V_{i_1, i_2}^{(l-1)} (\hat{H}_{i_1, 2}^{(l)})^T \hat{H}_{i_2, 2}^{(l)}
    \end{array}
    \right)_{i_1, i_2 = 1, ..., N}
    \\ = & \sum_{l = 1}^L
    \left(
    \begin{array}{ccc}
    	V_{i_1, i_2}^{(l-1)} 1_{m_L \times m_L}
    \end{array}
    \right)_{i_1, i_2 = 1, ..., N}
    \odot
    \left(
    \begin{array}{ccc}
    	(\hat{H}_{i_1, 2}^{(l)})^T \hat{H}_{i_2, 2}^{(l)}
    \end{array}
    \right)_{i_1, i_2 = 1, ..., N}
    \\ & \text{(where $\odot$ denotes pointwise multiplication, }
    \text{and $1_{m \times m}$ denotes an $m \times m$ matrix of all ones)}
    \\ = & \sum_{l = 1}^L
    \left( V^{(l-1)} \otimes 1_{m_L \times m_L} \right) \odot \left( (\hat{H}_2^{(l)})^T \hat{H}_2^{(l)} \right)
    \\ & \text{(let $\hat{H}_2^{(l)} =
    \begin{pmatrix}
    	\hat{H}_{1, 2}^{(l)} & \cdots & \hat{H}_{N, 2}^{(l)}
    \end{pmatrix}
    $, }
    \text{let $V^{(l)} =
    \left(
    \begin{array}{ccc}
    	V_{i_1, i_2}^{(l)}
    \end{array}
    \right)_{i_1, i_2 = 1, ..., N}
    $)}
\end{align*}

The cost of computing the above expression is $O(L m_L^2 N^2 + \sum_{l=1}^N \left(m_l m_L^2 N^2 + m_L^2 N^2 + m_{l-1} N^2 \right) ) = O(m_L^2 N^2 \sum_{l=1}^N m_l)$.

Then, we have Algorithm \ref{algo_9}.

\begin{algorithm}[h]
	\caption{Compute a $|S| \times |S|$ block matrix $\left(
J_{i_1} J_{i_2}^T
\right)_{i_1, i_2 \in S}$
}
	\label{algo_9}
	\begin{algorithmic}[1]
    	\STATE {\bf Input:} $\theta$, $h_i$, $v_i$ ($i \in S$), $S$
        \STATE {\bf Output:} $
\left(
J_{i_1} J_{i_2}^T
\right)_{i_1, i_2 \in S}
$

    	\FOR {$i \in S$}
        	\FOR {$j = 1, ..., m_L$}
            	\STATE $
                \left(
                \begin{array}{ccc}
                	\hat{h}_{i, 2}^{(l), (j)}
                \end{array}
                \right)_{l = 1, ..., L}
                = \text{Compute\_J\_transpose\_V}(e_j, \theta, h_i, v_i)$ (partly-computing mode)
                \STATE \COMMENT{see Algorithm \ref{algo_10}}
            \ENDFOR
        \ENDFOR

        \FOR {$l = 1, ..., L$}
        	\FOR {$i \in S$}
            	\STATE $\tilde{v}_i^{(l-1)} =
            	\begin{pmatrix}
            		v_i^{(l-1)}
                	\\ 1
            	\end{pmatrix}
            	$

        		\STATE $\hat{H}_{i,2}^{(l)} =
    			\begin{pmatrix}
    				\hat{h}_2^{(l), (1)} & \cdots & \hat{h}_2^{(l), (m_L)}
    			\end{pmatrix}
    			$
            \ENDFOR

            \STATE $\tilde{v}^{(l-1)} =
            \begin{pmatrix}
            	\tilde{v}_i^{(l-1)}
            \end{pmatrix}_{i \in S}
            $
            \COMMENT{arranged in a row}

        	\STATE $\hat{H}_2^{(l)} =
    		\begin{pmatrix}
    			\hat{H}_{i, 2}^{(l)}
    		\end{pmatrix}_{i \in S}
    		$
            \COMMENT{arranged in a row}

        	\STATE $B_l = \left( \left( (\tilde{v}^{(l-1)})^T \tilde{v}^{(l-1)} \right) \otimes 1_{m_L \times m_L} \right) \odot \left( (\hat{H}_2^{(l)})^T \hat{H}_2^{(l)} \right)$
        \ENDFOR

        \RETURN $\sum_{l=1}^L B_l$
	\end{algorithmic}
\end{algorithm}

\end{document}